\documentclass[twoside,11pt]{article}

\usepackage[abbrvbib]{jmlr2e}

\usepackage{graphicx}
\usepackage{amsmath}
\usepackage{hyperref}
\usepackage{bm}
\usepackage{amsfonts}
\usepackage{amssymb}
\usepackage{mathtools}
\usepackage{centernot}
\usepackage{enumerate}
\usepackage{algpseudocode}  
\usepackage{algorithm}
\usepackage{qtree}
\usepackage{stfloats} 

\usepackage[T1]{fontenc}    
\usepackage{url}            
\usepackage{booktabs}       
\usepackage{nicefrac}       
\usepackage{microtype}      
\usepackage{multirow}       

\algnewcommand{\LeftComment}[1]{\Statex \(\triangleright\) #1}

\jmlrheading{}{2020}{}{}{}{}{Jonathan D. Young, Bryan Andrews, Gregory F. Cooper, and Xinghua Lu}

\ShortHeadings{Learning Latent Causal Structures with a RINN}{Young, Andrews, Cooper, and Lu}
\firstpageno{1}

\begin{document}

\title{Learning Latent Causal Structures with a Redundant Input Neural Network}

\author{\name Jonathan D. Young \email jdy10@pitt.edu \\
       \addr Intelligent Systems Program\\
       University of Pittsburgh\\
       Pittsburgh, PA 15206, USA
       \AND
       \name Bryan Andrews \email bja43@pitt.edu \\
       \addr Intelligent Systems Program\\
       University of Pittsburgh
       \AND
       \name Gregory F. Cooper \email gfc@pitt.edu \\
       \addr Department of Biomedical Informatics\\
       University of Pittsburgh
       \AND
       \name Xinghua Lu \email xinghua@pitt.edu \\
       \addr Department of Biomedical Informatics\\
       University of Pittsburgh}

\editor{Thuc Duy Le, Lin Liu, Kun Zhang, Emre K{\i}c{\i}man, Peng Cui, and Aapo Hyv{\"a}rinen}

\maketitle

\begin{abstract}
Most causal discovery algorithms find causal structure among a set of observed variables.  Learning the causal structure among latent variables remains an important open problem, particularly when using high-dimensional data.  In this paper, we address a problem for which it is known that inputs \textit{cause} outputs, and these causal relationships are encoded by a causal network among a set of an unknown number of latent variables.  We developed a deep learning model, which we call a redundant input neural network (RINN), with a modified architecture and a regularized objective function to find causal relationships between input, hidden, and output variables.  More specifically, our model allows input variables to directly interact with all latent variables in a neural network to influence what information the latent variables should encode in order to generate the output variables accurately.  In this setting, the direct connections between input and latent variables makes the latent variables partially interpretable; furthermore, the connectivity among the latent variables in the neural network serves to model their potential causal relationships to each other and to the output variables.  A series of simulation experiments provide support that the RINN method can successfully recover latent causal structure between input and output variables.
\end{abstract}


\section{Introduction}

Causal discovery is a type of machine learning that focuses on learning causal relationships from data, often solely from observational data.  In this context, a causal structure is represented as a graph $G = (V,E)$ containing a set of vertices $V$ and a set of edges $E$, and the goal is to infer the data-generating causal structure from data.  The majority of contemporary causal discovery algorithms involve searching the space of possible structures to identify the structure supported by the data, and there have been numerous algorithms developed for this purpose \citep{cooper1999overview, spirtes2000causation, lagani2016probabilistic, maathuis2016review, peters2017elements, heinze2018causal}.  Most causal discovery algorithms find causal structure among the observed variables of a dataset.  Learning the causal structure among latent variables remains an important open problem, particularly when using high-dimensional data or attempting to leverage \textit{a priori} knowledge that one set of observed variables causally influences a non-overlapping (disjoint) set of observed variables.  Latent causal structure can provide additional important insight into the causal mechanisms of a system or process.  Thus, new approaches to discovering latent causal structure are needed.  In this work, we explore the utility of using deep learning models for finding latent causal structure when the data consist of two sets of variables and it is known \textit{a priori} that one set ("inputs") causes the other ("outputs")---and the causal path from inputs to outputs is mediated by a set of an unknown number of latent variables, among which the causal structure is also unknown.

This computational problem has important applications, including its application to cancer biology.  Cancer results from somatic genomic alterations (SGAs) (e.g., mutations in DNA) that perturb the functions of signaling proteins and cause aberrant activation or inactivation of signaling pathways, which often eventually influences gene expression.  Typically, an SGA only directly affects the function of a single signaling protein, and its impact is transmitted through a cascade of signaling proteins to influence gene expression.  Since the functional states of signaling proteins in pathways are usually not measured, the whole signaling system can be thought of as a set of hierarchically organized latent variables of which we would like to infer how changed states of some signaling proteins causally influence the state of others.  In a cancer cell, when one signaling protein is perturbed by an SGA event (an observed input variable), it may causally affect the functional states of the proteins in a pathway (which are not observed) and eventually result in changed gene expression (observed output variables).  Thus, our task is to learn how a set of observed input variables causally influence a set of observed output variables, through signals transmitted among a set of latent variables.

In contrast to traditional causal discovery methods, here we explore a modified deep learning model to learn latent causal relationships.  Deep learning represents a family of machine learning algorithms or strategies that originated from artificial neural networks (ANN).  ANNs represent a framework, loosely inspired by biological neurons, for learning a function (represented as a set of parameters or weights) that maps inputs to outputs.  An essential characteristic of deep learning models is their ability to learn compositional representations of the data \citep{lee2008sparse,lecun2015deep}.  A deep learning model is composed of multiple layers of latent variables (hidden nodes or units) \citep{lecun2015deep, goodfellow2016deep}, which learn to represent the complex statistical structure embedded in the data, such that different hidden layers capture statistical structure of different degrees of complexity.  Researchers have found that deep learning models can represent the hierarchical organization of signaling molecules in a cell, with latent variables as natural representations of unobserved activation states of signaling molecules \citep{chen2016learning, young2017unsupervised, lu2018novel, tao2020genome}. 

Conventional neural networks behave like black boxes in that it is difficult to interpret what signal a hidden node in the network encodes.  We hypothesize that with certain modifications inspired by the biological problem mentioned above, one may learn a partially interpretable deep neural network so that the relationships between latent variables can be interpreted as part of a causal chain from input to output variable.  To this end, we designed a modified deep neural network, the redundant input neural network (RINN) that allows each input variable to directly connect to each latent variable within the deep learning hierarchy.  This redundant input structure allows one to learn direct causal relationships between an input variable and \textit{any} latent variable.  This is important for learning cellular signaling pathways, as SGAs perturb signaling proteins throughout the pathway hierarchy, not just at the beginning.  Also, without this modification any attempt at interpreting the latent variables (by mapping to input or output space) would be more complex.  We also developed a robust pipeline for testing an algorithm's ability to capture latent causal structure, including causal simulated data inspired by cellular signaling pathways, a method for visualizing the weights of a neural network, and a method for measuring precision and recall of the causal structure.  In addition, we compare the RINN with a conventional feedforward deep neural network (DNN), a deep belief network (DBN), a RINN optimized with a constrained evolutionary strategy, and the DM algorithm (a causal discovery algorithm).

\section{Related Work}
The work in this study concentrates on finding latent causal relationships when the causal direction between inputs and outputs is known.  The DM (Detect MIMIC (Multiple Indicators Multiple Input Causes)) algorithm is a causal discovery algorithm for finding latent causal relationships between inputs and outputs \citep{murray2014dm, murray2015going}.  The DM algorithm uses a series of heuristics based on conditional independence and Sober's criterion \citep{sober1998black, murray2014dm} (see Appendix) to identify the latent causal structure.  One limitation of the DM algorithm is that it constrains each latent variable to be adjacent to at least one input and one output variable.  This limits the hierarchical and compositional relationships that DM can identify, which makes it less accurate for cellular signaling pathway discovery.  In addition to the DM algorithm, there are other algorithms developed to find latent structure, including factor analysis algorithms, algorithms that use variations of the Expectation-Maximization (EM) algorithm to find latent structure \citep{friedman1997learning, elidan2005learning}, and algorithms based on fast causal inference (FCI) \citep{spirtes1995causal,colombo2012learning}. These algorithms are, in general, highly constrained, intractable in high-dimensional spaces, return only a subset of the latent causal relationships, \textit{or} don't leverage prior knowledge that input variables cause output variables.

Recently, there has been an increased interest in the deep learning community to combine deep learning with causal discovery.  \cite{kalainathan2018sam} developed an algorithm named structural agnostic model (SAM) that trains multiple generative adversarial networks (GAN) \citep{goodfellow2014generative}, one GAN to generate each variable in $\bm{X}$, the measured variables or features in a dataset.  Each GAN uses $\bm{X}_{-j}$ variables (i.e., all variables in $\bm{X}$ other than $x_j$) multiplied by a weight mask, $\bm{a_j}$, to generate $x_j$.  All of the $\bm{a_j}$ vectors, once learned, describe the learned causal structure.  Somewhat similar to SAM, \cite{ke2019learning} also used boolean masks applied to inputs in an ensemble of neural networks to model the causal relationship between observed variables and the observed variable's parents.  \cite{lopez2017discovering} used supervised CNNs to predict the direction of the causal arrow between two variables in static images and \cite{harradon2018causal} built a Bayesian causal model from lower dimensional representations (of images) learned by an autoencoder CNN.  Some research in this space has even concentrated on counterfactual inference or individual treatment effect (ITE) prediction \citep{johansson2016learning, hartford2017deep, louizos2017causal, shalit2017estimating}.  

Experimenting with different architectures is a common theme in deep learning research.  Various versions of fully-connected neural networks (i.e., each node is connected to every other node in the network) have been studied since 1990 \citep{fahlman1990cascade, wilamowski2010neural}.  The dense convolutional network (DenseNet) \citep{huang2017densely} is a fully-connected neural network that concatenates all previous feature-maps (i.e., hidden layer outputs) into a single tensor and use this tensor as input to the next hidden layer in the network.  Highway networks \citep{srivastava2015training} and deep residual networks (ResNets) \citep{he2016deep} are deep learning models with "skip" connections, allowing the output of a previous hidden layer to directly influence the output of a future hidden layer, after skipping one or more hidden layers.  The RINN is perhaps most closely related to a recurrent neural network (RNN).  The RINN model can be thought of as an RNN with non-shared weights, time-invariant (i.e., static) input, and output generated only at the last time-step \citep{liao2016bridging}.

\section{Redundant Input Neural Network (RINN)}
\label{sec:strategies_rinn}

In this study, we desired to model inputs that directly connect to any hidden layer in a neural network.  To accomplish this, we developed the redundant input neural network (RINN) (Figure \ref{fig:rinn} and Algorithm \ref{alg:strategies_rinn}). The RINN has an extra copy of the input $\bm{x}$ directly connected to all hidden layers after the first hidden layer.  In contrast to the RINN, the inputs of a conventional feedforward DNN are only connected to the first hidden layer.  Just like a DNN, the RINN is trained through backpropagation and stochastic gradient descent.  Each hidden layer of a RINN with redundant inputs is calculated according to:

$$\bm{h}^{(i)} = \phi( ( [\bm{h}^{(i-1)},\bm{x}] \bm{W}_i ) + \bm{b_i})$$

\noindent where $\bm{h}^{(i-1)}$ represents the previous layer's output vector, $\bm{x}$ is the vector input to the neural network, $[\bm{h}^{(i-1)},\bm{x}]$ represents concatenation into a single vector, $\bm{W}_i$ represent the weight matrix between hidden and redundant nodes in layer $i-1$ and hidden layer $i$, $\phi$ is a nonlinear function (e.g., ReLU), and $\bm{b_i}$ represents the bias vector for layer $i$.  In contrast to a RINN, a plain DNN calculates each hidden layer as: 
$$\bm{h}^{(i)} = \phi( ( \bm{h}^{(i-1)} \bm{W}_i )+ \bm{b_i}).$$

\begin{figure}[tb]
\centering 
\includegraphics[scale=0.85]{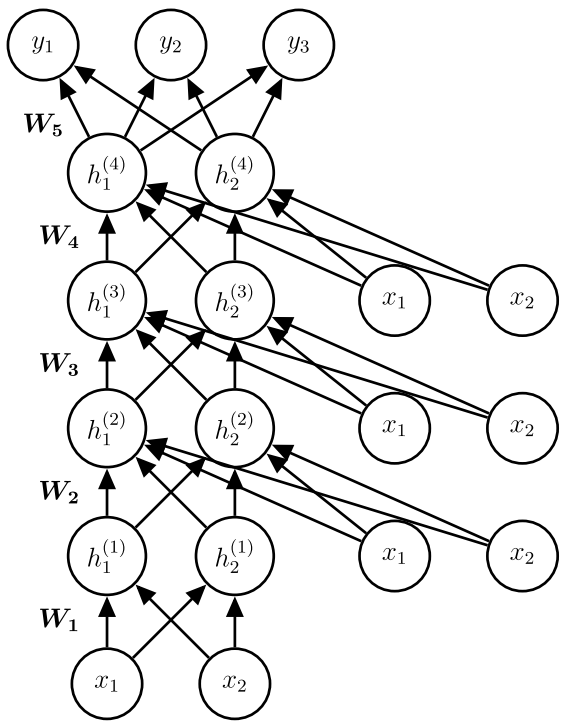}
\caption{Redundant input neural network (RINN).  A RINN with four hidden layers ($h^{(1)}$, $h^{(2)}$, $h^{(3)}$, $h^{(4)}$, each with two nodes), two inputs ($x_1, x_2$ on bottom), three outputs ($y_1$, $y_2$, $y_3$), and three sets of redundant inputs ($x_1$, $x_2$ on right side).  Each node represents a scalar value and each edge represents a scalar weight. The weights between layers are collected in weight matrices $\bm{W_1}$, $\bm{W_2}$, $\bm{W_3}$, $\bm{W_4}$, and $\bm{W_5}$.}
\label{fig:rinn}
\end{figure}

\begin{algorithm}[tb]
\small
\caption{Redundant Input Neural Network (RINN)}
\label{alg:strategies_rinn}
\begin{algorithmic}
\State Bias vectors omitted for clarity
\State $\bm{W}_{l,a}$ is weight matrix between hidden layer $l-1$ and hidden layer $l$
\State $\bm{W}_{l,b}$ is weight matrix between input $l-1$ and hidden layer $l$
\State $\bm{h}_l$ is the vector of values representing hidden layer $l$
\State $\bm{x}_i$ is the input vector for instance $i$
\State $\bm{y}_i$ is the output vector for instance $i$
\State \textit{n} is the number of samples in the dataset
\State $ERROR$ is the objective function to be optimized
\State
\For{$i=1$ to $n$}
	\State $\bm{a}_1 = ReLU(\bm{x}_i \bm{W}_{0,b})$
	\For{$j=1$ to $num\_hid\_layers$}
	    \State $\bm{a}_j =$ CONCATENATE($\bm{a}_{j}, \bm{x}_i$)
	    \State $\bm{W}_j =$ CONCATENATE($\bm{W}_{j,a},\bm{W}_{j,b} $)
	    \State $\bm{a}_{j+1} = ReLU(\bm{a}_j \bm{W}_j)$ 
	\EndFor
    \State $\hat{\bm{y}_i} = sigmoid(\bm{a}_{j+1} \bm{W}_{j+1,a})$ \Comment{for binary $\bm{y}$}
    \State Update all $\bm{W}$ by descending their gradient: 
    $$\nabla_{\bm{W}} ERROR(\bm{y}_i, \hat{\bm{y}_i})$$
\EndFor
\end{algorithmic}
\end{algorithm}

In order to use a neural network to capture causal structure in the network's weights, the weights need to be regularized or constrained in some way.  Without regularization, virtually all weights in a neural network will be nonzero after training, meaning that the returned causal structure (after interpreting nonzero weights as edges) would be much too dense.  We evaluated both $L_1$ and $L_2$ regularization and $L_1$ provided far superior results for finding the causal structure of our ground truth DAG (results not shown).  The objective functions for all neural network-based strategies used in this study included $L_1$ regularization of the weights. \textit{Therefore, for binary outputs, binary cross-entropy loss plus $L_1$ regularization of the weights was used as the objective function.  For non-binary outputs, mean squared error (MSE) plus $L_1$ regularization of the weights was used.} See the Appendix for a discussion of the model selection techniques used in this work, i.e., how to find the hyperparameters (e.g., learning rate, regularization rate, sizes of hidden layers, etc.) for the best models.

\section{Simulating Data From a Known Causal Structure}
\label{sec:methods_sim_data}
Ultimately, we are interested in biological cellular signaling systems and constructing these causal signaling pathways from mutation and expression data.  However, accurate and complete ground truth cellular signaling pathway knowledge (i.e., causal structure) is not available.  To evaluate how well a neural network approach can discover causal structure through its weights, we generated simulated data from an artificial causal hierarchical structure.  We first manually created a ground truth directed acyclic graph (DAG), $G_T$, that fit the requirements described above (i.e., hierarchical structure between a set of inputs and outputs).  In $G_T$ (Figure \ref{fig:ground_truth_big}a), the input layer is directly connected to both the first and second hidden layers to allow us to recover direct causal relationships between inputs and hidden variables.  This structure is also the most biologically plausible way that mutations would affect biological entities in a signaling pathway (assuming the latent variables to be biological entities), as mutations can have an effect at any location in the signaling pathway hierarchy, not just at one level (i.e., hidden layer) of the hierarchy.  The ground truth DAG had 16 inputs $\bm{x}$, 16 outputs $\bm{y}$, and two hidden layers $\bm{h_1}$, $\bm{h_2}$ (Figure \ref{fig:ground_truth_big}a).  We needed an easy way to visualize $G_T$, so we connected the nodes in a specific way to be able to generate visibly recognizable patterns (relative to output space) for each node in $G_T$ (Figure \ref{fig:ground_truth_big} and Section \ref{sec:methods_visualizing}).  Given a specific node $x_i$ in $G_T$, these patterns (heatmaps) indicate if an output node is affected by the value of $x_i$, i.e., \textit{Which output nodes have $x_i$ as an ancestor?} (Figure \ref{fig:ground_truth_big}c, \ref{fig:ground_truth_big}d).

\begin{figure}[tb]
\centering
\includegraphics[scale=0.92]{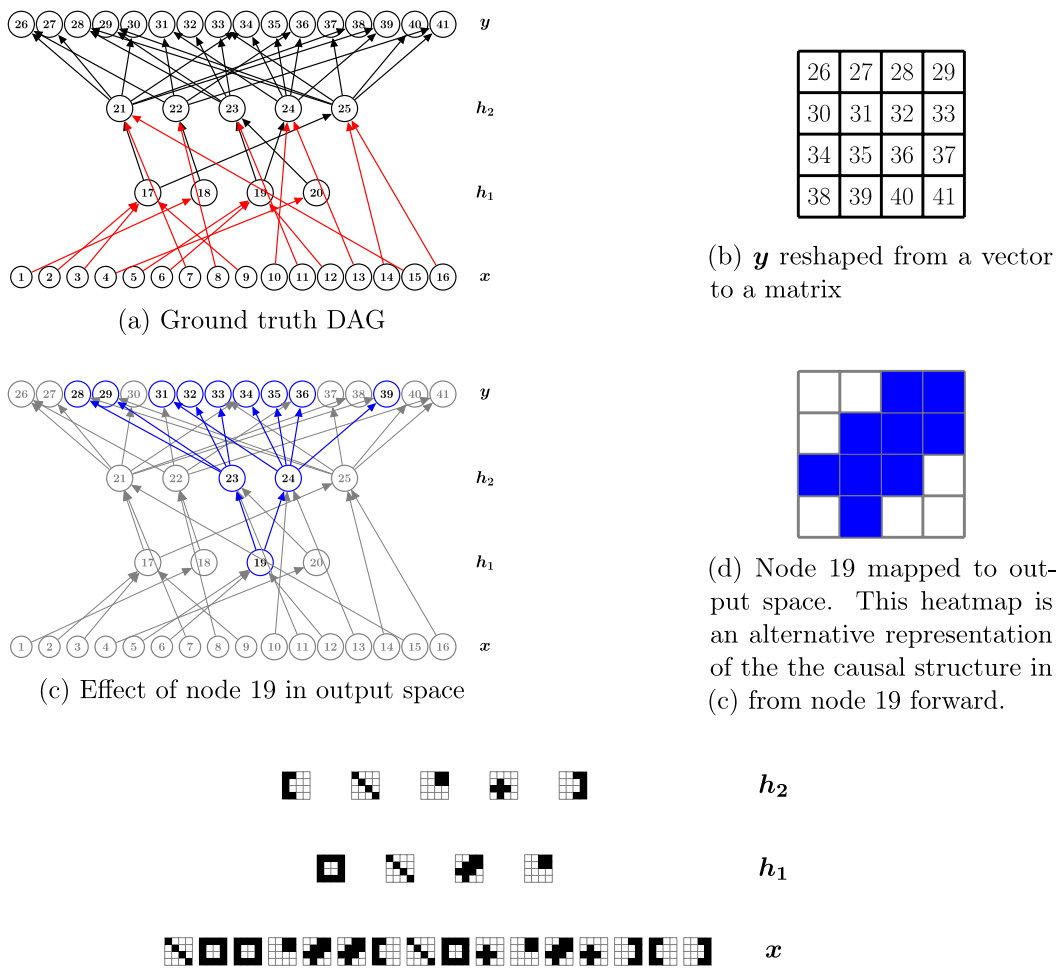}
\caption{Ground truth DAG for simulated data and visualizing output space.}
\label{fig:ground_truth_big}
\end{figure}

We constructed $G_T$ with inputs $\bm{x}$ representing the presence or absence of DNA alterations (i.e., SGAs) and the outputs $\bm{y}$ representing differentially expressed genes (DEGs).  Mutation and expression data are readily available from many sources, e.g., The Cancer Genome Atlas (TCGA) \cite{weinstein2013cancer}.  Between these input and output layers, we envisioned a biological signaling pathway with the hidden layer $\bm{h}_1$ closest to the inputs representing major biological pathways and $\bm{h}_2$ representing transcription factors.  Biologically, mutations (i.e., input nodes $\bm{x}$) directly target transcription factors and proteins higher up in the hierarchy (e.g., cell membrane receptors).  To reflect this in $G_T$, we allowed input mutations in $\bm{x}$ to directly connect to nodes in $\bm{h}_1$ or $\bm{h}_2$.  Each mutation affected only one node, but each node could be affected by multiple mutations.  This is similar to how biological mutations function.  $G_T$ is a crude representation of a biological signaling pathway, and one obvious omission in the DAG abstraction are feedback cycles, which are prevalent in biological signaling pathways.  We generated six different simulated datasets.  Datasets 1-4 had 16 inputs, 16 outputs, and 5,000 instances.  Datasets 5 and 6 had 650 inputs, 84 outputs, and 5,000 instances.

\subsection{Simulated Data}
\begin{enumerate}[1.]
\item  \textbf{Dataset 1:  Matrix Multiplication with Interventions---no noise, binary input, positive integer-valued output}. We simulated data from $G_T$ by having the inputs \textit{intervene} on their adjacent hidden nodes.  An "active" (i.e., 1) input intervention node would set the value of an adjacent hidden node to $0$ for that sample---intended to capture the effect of a biological loss-of-function mutation on a protein. Please see Algorithm \ref{alg:methods_matrix_mult} in the Appendix. 

\item \textbf{Dataset 2:  Linear Gaussian SEM---high noise, input and output $\bm{\in \mathbb{R}}$}.

$$ \bm{b} \sim \mathcal{N}(0, \sigma^{2}) \hspace{5mm}
 \sigma^{2} \sim \mathcal{U}(1,3) \hspace{5mm}
	\bm{W} \sim \pm \mathcal{U}(0.5, 1 .5) $$
Let $G_T$ be our ground truth DAG with nodes $\bm{N}$ (input, hidden, and output), $n_j^{l}$ be node $j$ at layer $l$, $x_i^{l-1}$ be the value of parent (i.e., direct cause) $i$ in layer $l-1$, $w_{ij}^{l}$ be the weight value from parent $i$ to node $j$, and $b_j^{l}$ be the bias for node $j$.  Then we simulate data using a structural equation model (SEM) to calculate the value of each node $n_j^{l}$ in $G_T$ as:
$$n_j^l = \sum_{i \in parents_{G_T} (n_j^l)} w_{ij}^{l}x_i^{l-1} + b_j^{l}$$

\item \textbf{Dataset 3:  OR Logical Operator---high noise, binary}.
We modeled the value of a node in $G_T$ as the output of a logical OR applied to the values of the parents of said node.  We sampled a probability, $p$, to represent: 
$$P(x=1|parents_{G_T}(x))$$ 
for all possible binary combinations of a node's parents' values.  $p$ was sampled from Beta distributions peaked close to $0.10$ or $0.90$ depending on if an OR operator applied to the values of the parents evaluated to False or True, respectively.  For example, to model an OR operator with two parents, we would sample four $p$ values, each representing a Bernoulli distribution as follows,

$$P(x = 1|Pa_{1, G_T} = 1,Pa_{2, G_T} = 1) \sim Beta(90,10)$$
$$P(x = 1|Pa_{1, G_T} = 0,Pa_{2, G_T} = 1) \sim Beta(90,10)$$
$$P(x = 1|Pa_{1, G_T} = 1,Pa_{2, G_T} = 0) \sim Beta(90,10)$$
$$P(x = 1|Pa_{1, G_T} = 0,Pa_{2, G_T} = 0) \sim Beta(10,90)$$

\noindent where $Pa_{1, G_T}$ indicates parent $1$ of node $x$ in graph $G_T$.  This sampling from Beta distributions provided a collection of Bernoulli distributions for all possible binary combinations of parents.  We then used these Bernoulli distributions to generate noisy data (noisy because $p$ not equal to 0 or 1) from OR operators.  Please see Algorithm \ref{alg:and_or_xor} in the Appendix, but assume all OR operators.

\item \textbf{Dataset 4:  AND, OR, XOR Logical Operators---moderate noise, binary}.  To increase complexity, we added AND and XOR operators to replace some of the OR operators in Dataset 3 and re-simulated the data following a similar procedure as for Dataset 3 (see Algorithm \ref{alg:and_or_xor} in the Appendix).  In contrast to Dataset 3, we sampled $p$ from Beta distributions peaked close to $0.05$ and $0.95$ for Dataset 4, meaning that Dataset 4 is less noisy than Dataset 3. 

\item \textbf{Dataset 5:  TCGA + OR---binary}. We wanted to test our algorithms on higher-dimensional, biologically-related datasets.  Biological datasets with mutation and expression data, and a robust, known ground truth causal structure, to our knowledge, do not exist.  So we embedded simulated data (Dataset 3---OR) within a larger biological dataset from TCGA to see if $G_T$ could be discovered, and to also evaluate how the different algorithms performed simply on the prediction task for this biologically-related dataset.  To accomplish this, we obtained mutation and expression data from TCGA and used an algorithm from \citep{cooper2018tumor, cai2019tci} called TCI (tumor-specific causal inference) to perform feature selection on both the mutation and expression data, ultimately leading to a binary dataset with 634 alterations (inputs) and 68 DEGs (outputs) for 5,000 tissue samples.  Next, we embedded simulated Dataset 3 (OR operator) into the feature-reduced TCGA data.  This generated a final dataset with 650 inputs (SGAs) and 84 outputs (DEGs) for 5,000 instances.  Please see \citep{cooper2018tumor, cai2019tci} for a detailed explanation of how TCI works.

\item \textbf{Dataset 6:  TCGA + OR, XOR, AND---binary}. The same as Dataset 5, but embedded with Dataset 4 (instead of Dataset 3).

\end{enumerate}

\section{Other Strategies}
\subsection{DNN}
\label{sec:strategies_dnn}
We compared the RINN with a feedforward DNN. A DNN learns a function mapping inputs ($\bm{x}$) to outputs ($\bm{y}$) according to:
$$f(\bm{x}) = \phi( \cdots \phi(\phi(\bm{x} \bm{W}_1) \bm{W}_2) \cdots  \bm{W}_n) = \hat{\bm{y}}$$
\noindent where $\bm{W}_i$ represent the weight matrices between layers of a neural network, $\phi$ is a nonlinear function (e.g., sigmoid), and $\hat{\bm{y}}$ is the predicted output.  DNNs also have bias vectors added at each layer, which have been omitted from above for clarity.  For more details about DNNs, see \cite{goodfellow2016deep}. 

\subsection{DBN}
\label{sec:strategies_dbn}
A deep belief network or DBN \citep{hinton2006reducing} is a type of autoencoder, an unsupervised DNN.  The DBNs used in this paper were trained using only $\bm{y}$ data, i.e., DBNs learned to map $\bm{y}$ to a lower dimensional representation of $\bm{y}$ (in the hidden layers) and then to reconstruct $\bm{y}$ from the lower dimensional representation.  Error is measured by comparing $\bm{y}$ to $\hat{\bm{y}}$.  We evaluated the trained "decoder" network only (i.e., the second half of the network going from lowest dimensional representation to $\bm{y}$) for evidence of $G_T$ latent variables.  We had difficulty recovering any of the causal ground truth graph (Figure \ref{fig:ground_truth_big}a) with a default DBN, but adding $L_1$ regularization to the finetuning step objective function allowed us to recover at least some of the causal structure. 

\subsection{Constrained Evolutionary Strategy (ES-C)}
We compared the conventional RINN (using gradient descent optimization) to a RINN using an evolutionary algorithm for weight optimization (Algorithm \ref{alg:ga}).  We developed an evolutionary strategy (ES) to optimize the weights of a RINN to explore optimization strategies that make it easier to interpret the weights of a neural network as causal relationships.  Specifically, we used a constrained evolutionary strategy (ES-C), where we constrained the possible values for the weights (e.g., $weight \in \{-1, -0.5, 0, 0.5, 1 \}$).  Constraining the weights in this way represents a non-differentiable objective function, meaning that gradient descent optimization cannot be used.  We represented an individual in the population to be evolved as a vector of weights.  This vector of weights represents a RINN.  Initially these weights were set to all zeros, then if a weight was selected for mutation (as dependent on a mutation rate), it was randomly changed to a weight in the range of legal weight values.  For a fitness function, we used $L_1$ regularization plus either binary cross-entropy error or squared error depending on the dataset.  ES-C was not set up to be parallelized across processors, as early prototyping experiments suggested that ES-C would not perform nearly as well as RINN with gradient descent optimization.  See the Appendix for more information.

\begin{algorithm}[tb]
\small
\caption{Constrained Evolutionary Strategy for Neural Network Weight Optimization}
\label{alg:ga}
\begin{algorithmic}
\State $\bm{f}$ is a vector of fitness values
\State $\bm{f}_{e}$ is a vector of fitness values for elites; initially zeros
\State $\bm{Data}$ represents all input and output data
\State $\bm{w}_i$ is all the weights of neural network $i$ reshaped into a vector
\State $\bm{W}$ is a matrix of all $\bm{w}_i$
\State $\bm{W}_{e}$ is a matrix of all $\bm{w}_i$ for elites only; initially zeros
\State $pw$ is the set of possible weight values (e.g., $\{-1, 0, 1\}$) 

\State
\For{$i=1$ to $num\_generations$}  \Comment for each generation
    \For{$j=1$ to $population\_size$}  \Comment for each neural network
        \State $f_j = $ FITNESS($\bm{w}_j$, $\bm{Data}$) \Comment neural network objective function
    \EndFor
    \State $\bm{f} = $ CONCATENATE($\bm{f}, \bm{f}_{e}$) \Comment add previous generation elite fitness values
    \State $\bm{W} = $ CONCATENATE($\bm{W}, \bm{W}_{e}$) \Comment add previous generation elite weights
    \State $elites =$ TOP\_20\_PERCENT($\bm{f}$)  \Comment get best neural networks
    \State $\bm{f}_e =$ fitness values of the elite networks
    \State $\bm{W}_e =$ weight values of the elite networks
    \For{$k=1$ to $population\_size$}  
        \State $\bm{w}_k = $ MATE(RAND($elites$), RAND($elites$))  \Comment 50\% of weights from each, randomly
        \State $\bm{w}_k = $ MUTATE($\bm{w}_k$, $pw$)  \Comment randomly change some $w_k$ to different value in $pw$
    \EndFor
\EndFor
\State RETURN $\bm{w}_{best}$ \Comment neural network with highest fitness function
\end{algorithmic}
\end{algorithm}

\subsection{DM Algorithm}
As another comparison, we evaluated the performance of the DM (Detect MIMIC) algorithm \citep{murray2015going} on the first four simulated datasets.  The DM algorithm takes two tuning parameters: an alpha level for Fisher's Z test of conditional independence and an alpha level for Sober's criterion.  We performed an abbreviated grid search over these tuning parameters.  We selected values for the first alpha level $\alpha \in \{0.05, 0.01, 0.001, 0.0001\}$ motivated by \cite{ramsey2017comparison} and values for the second alpha level $\alpha \in \{ 0.1, 0.05, 0.01 \}$.  Results for the best performing combinations of tuning parameters for the DM algorithm were reported. See the Appendix for more information.

\section{Causal Assumptions}
\label{sec:methods_causal}

Given a dataset where the variables may be partitioned into inputs $X$ and outputs $Y$, fitting the RINN to the data results in a sparsely connected DAG $G_E = (V_E, E_E)$. Generally, the structure of a neural network has no causal meaning, however, in this section we list a set of assumptions that connect the hidden nodes and directed edges of $G_E$ to latent causal relationships. That is, we fit the RINN to data for the purpose of learning genuine latent variables and causal pathways between $X$ and $Y$ mediated by those latent variables.

Before stating our assumptions, we note the following conventions and definitions. Let $G_T = (V_T, E_T)$ denote the true causal graph where $V_T$ may be partitioned into $X$, $Y$, and the latent variables $L_T$. Let $G_M = (V_M, E_M)$ denote the modified graph where $V_M$ may be partitioned into $X$, $Y$, and the latent variables $L_M$. We construct $G_M$ from $G_T$ by: (i) removing any latent variable that has a descendant in $X$ or an ancestor in $Y$; (ii) removing any latent variable that does not have an ancestor in $X$ or a descendant in $Y$; (iii) removing any edge between two members of $X$ or two members of $Y$. Intuitively, the modification procedure deletes variables and edges that are not of interest and retains identifiable variables and causal pathways between $X$ and $Y$. We say a latent variable is identifiable if: (i) it has no descendants in X or ancestors in Y; (ii) its ancestors restricted to $X$ are non-empty and unique\footnote{By unique, we mean a set that is \textit{not equal} to any other set under consideration.} with respect to the other latent variables satisfying (i); (iii) its descendants restricted to $Y$ are non-empty and unique with respect to the other latent variables satisfying (i).

\begin{enumerate}[{A}1.]

\item The true causal graph $G_T$ is a DAG where the set of post-modification latent variables $L_M$ is the subset of latent variables that are identifiable in $L_T$.

In a biological context, feedback cycles are a common characteristic of molecular signaling pathways. Assuming acyclicity is a limitation of using the methods in this study.

\item For all $w \in L_T$, $x \in X$, and $y \in Y$, $y$ is not an ancestor of $w$ or $x$, and there is at least one latent variable on every path between $x$ and $y$ in $G_T$. That is, the inputs cause the outputs mediated through latent variables.

For example, mutations in DNA cause changes in gene expression through modifications to cell signaling.  The methods developed in this study can be used with any datasets where the causal direction between two sets of variables is known. 

\item For all $w \in L_T$ and $y \in Y$, if there is a directed path from $w$ to $y$ that does not include some $x_1 \in X$ in $G_T$, then there is no directed path from $w$ to any $x_2 \in X$. That is, the inputs and outputs are unconfounded. 

In a biological context, where $X = SGA$ and $Y = DEG$, this assumption is generally plausible.

\item For all $x_1, x_2$ $\in X$, $x_1$ is not an ancestor of $x_2$ and for all $y_1, y_2$ $\in Y$, $y_1$ is not an ancestor of $y_2$. This assumptions still allows latent confounding between the members of $X$ and latent confounding between the members of $Y$.

\item No selection bias occurred during the data collection process.

\item Restricted Causal Markov: For conditional independence statements of the form ${x \perp \!\!\! \perp y \mid S}$ for $S \subseteq (X \cup Y) \setminus \{x, y\}$ where $x \in X$ and $y \in Y$, \textit{d}-separation\footnote{\textit{d}-separation is a criterion for reading conditional independence statements from a directed graph.} in $G_T$ implies conditional independence in the data distribution.

\item Restricted Causal Faithfulness: For conditional independence statements of the form ${x \perp \!\!\! \perp y \mid S}$ for $S \subseteq (X \cup Y) \setminus \{x, y\}$ where $x \in X$ and $y \in Y$, conditional independence in the data distribution implies d-separation in $G_T$. 

\item The identifiable latent structure of $G_T$ is minimal in the following sense. Let $G_T'$ be the graph resulting from altering\footnote{By altering, we mean adding or removing variables and edges such that the resulting graph differs by more than a relabeling of the variables or a sequence of covered edge reversals.} the identifiable latent structure. If $G_T'$ satisfies restricted Markov (see A6), then the altered modified latent edges $E_M'$ satisfy $|E_M'| > |E_M|$.

This can be seen as an appeal to Occam's razor---preferring the most parsimonious structure. 

\item The ground truth causal relationships can be modeled with one or multiple affine transformations plus nonlinear functions (hidden layer calculations in a neural network).

\item The global optimum achieved by the RINN results in a DAG $G_E$ satisfying restricted Markov and restricted faithfulness (see A6 and A7) whose latent structure is minimal (see A8). 

The RINN uses $L_1$ regularization to encourage learning the most parsimonious graph capable of representing the conditional distribution.

\end{enumerate}

\begin{lemma}
Let $L_R = L_T \setminus L_M$ be the latent variables removed during modification and $\pi$ be a path in $G_T$ between $x \in X$ and $y \in Y$. The following are equivalent:
\begin{enumerate}
\item $\pi$ is inducing relative to $L_R$;
\item $\pi$ is primitively inducing.\footnote{A path $\pi$ between variables $a$ and $b$ is \textit{inducing} relative to a set $L$ if every non-endpoint variable on $\pi$ is in $L$ or a collider, and every collider on $\pi$ is an ancestor of $a$ or $b$. When $L = \emptyset$, $\pi$ is \textit{primitively inducing}.}
\end{enumerate}
\end{lemma}

\begin{proof}
If $\pi$ is primitively inducing, then $\pi$ is inducing relative to $L_R$ by the definition of inducing path. Suppose by way of contradiction that $\pi$ is not primitively inducing, but $\pi$ is inducing relative to $L_R$. Then every non-collider on $\pi$ is a variable $v \in L_R$. Pick the closest $v$ to $x$ so that there are no non-colliders between $v$ and $x$ and consider the criteria for membership to $L_R$.

Suppose $v$ does not have an ancestor in $X$. Then $v$ is a parent of $w \in X$ on one side and directed into or out of $v$ on the other side. If directed into $v$, then there is a directed path from $y$ to $v$. This violates A2. If directed out of $v$, then there is a directed path from $v$ to a collider $u \in X$ followed by $y$ or from $v$ to $u \in Y$. This violates A2 or A3. Suppose $v$ does not have a descendant in $Y$. Then $v$ is adjacent to $x$ on one side and directed into or out of $v$ on the other side. If directed into $v$, then there is a directed path from $y$ to $v$. This violates A2. If directed out of $v$, then there is a directed path from $v$ to a collider $w \in X$ followed by either $y$ or $u \in L_R$. The former violates A2. In the latter case, repeat this argument with $u$ in place of $v$ and $w$ in place of $x$.

Suppose $v$ has a descendant $w \in X$. Then $v$ is adjacent to $u \in X$ on one side and directed into or out of $v$ on the other side. If directed into $v$, then there is a directed path from $y$ to $v$. This violates A2. If directed out of $v$, then there is a directed path from $v$ to $t \in Y$ or from $v$ to $t \in X$. The former violates A3. In the latter case, $t$ is followed by either $y$ or $s \in L_R$. The former violates A2. In the latter case, repeat this argument with $s$ in place of $v$ and $t$ in place of $x$. Suppose $v$ has an ancestor in $w \in Y$. Then A2 is violated.
\end{proof}

Supposing that a specified set of variables is latent, inducing paths characterize when edges should be added between variables. Therefore, Lemma 1 shows that no edges are added when performing modification steps (i) and (ii).

\begin{proposition}
The global optimum achieved by the RINN is the modified graph of the true causal graph, that is $G_E = G_M$.
\end{proposition}

\noindent \textbf{Proof Sketch}\hspace{1mm}  
Perform modification steps (i) and (ii) on $G_T$; by Lemma 1, the resulting graph satisfies restricted Markov and restricted faithfulness. By A4, no two members of $X$ are adjacent and no two members of $Y$ are adjacent. Perform modification step (iii) to no effect. Therefore the resulting graph satisfies restricted Markov and restricted faithfulness. By A8 and A10, $G_E$ and $G_M$ are equivalent up to a covered edge reversal. By A1 and A3, for all $x \in X$, $y \in Y$, and $w \in L_M$, $x$ is not a descendant of $y$ or $w$ and every $w$ has an ancestor in $X$. Therefore, using the definition of identifiable, we may propagate the orientations from $X$ forward making every edge orientation invariant. Accordingly, $G_E = G_M$.
\hfill\BlackBox\\[2mm]

The causal ground truth DAG in Figure \ref{fig:ground_truth_big}a and simulated data adhere to these assumptions. Given our assumptions, Proposition 2 states that the hidden nodes and directed edges learned by the RINN represent the latent structure over the identifiable latent variables in the true causal graph if the RINN reaches the global optimum. We conjecture that this result holds in practice for local optimum given adequate model selection. The empirical results of this paper provide support for this conjecture.

\section{Balancing Sparsity and Error}
\label{sec:methods_distance}
In order to find simple and accurate models, we needed to modify our approach to model selection in order to best balance sparsity of weights and prediction error.  After training a model with various combinations of hyperparameters, we selected the models (sets of weight matrices for each trained network) that were relatively sparse and had a relatively low prediction error.  Under-regularizing leads to overfitting the data, while over-regularizing leads to underfitting the data.  We hypothesized that the models with best chances of capturing causal relationships in their weights will be the models that optimally balance both prediction error and sparsity.  To balance this trade-off, in plots of prediction error versus sparsity (See Appendix Figure \ref{fig:ms}), we measured the Euclidean distance from each point (model) to the origin according to: 
$$d_{x} = \sqrt{ \left( \sum_{i=1}^{m} \sum_{j=1}^{r_i} \sum_{k=1}^{c_i} \lvert w_{j,k}^{(i)}\rvert \right) ^{2} + {L_{x}}^{2} } $$

\noindent where $L_{x}$ is the average validation set loss for neural network $x$, $m$ is the number of weight matrices in neural network $x$, $r_i$ and $c_i$ are the number of rows and columns in matrix $i$ respectively, and $w$ is a scalar weight.  Prior to calculating $d_x$, we removed values outside the 95\% quantile for each axis and scaled the axes using min-max scaling. The models with the lowest $d_x$ were selected for further analysis and calculation of precision and recall of causal structure.

\section{Visualizing a Neural Network}
\label{sec:methods_visualizing}
To increase our understanding of what the weights of a neural network capture after training, we developed a simple visualization technique based on matrix multiplication as follows:

$$\bm{M}_j = \bm{W}_j \bm{W}_{j+1} \bm{W}_{j+2} \cdots \bm{W}_h$$
$$ \bm{v}_j = reshape(\bm{M}_j) $$

\noindent where $\bm{M}_j$ is an $m \times n$ matrix for layer $j$, $m$ is the number of nodes in layer $j$, $n$ is the number of nodes in the output $\bm{y}$, $\bm{W}_j$ is the weight matrix between layer $j$ and $j+1$, and $\bm{W}_h$ is the last weight matrix.  Each row in $\bm{M}_j$ represents what a single node in layer $j$ maps to in output space. $reshape$ means to reshape each row of $\bm{M}_j$ to a $4 \times 4$ matrix (assuming $n = 16$) (Figure \ref{fig:ground_truth_big}b). Therefore, $\bm{v}_j$ is a list of $4 \times 4$ matrices (heatmaps) for all hidden nodes in layer $j$.  Next, the magnitude of each value in the $4 \times 4$ matrices is interpreted as a pixel intensity (or heatmap intensity), and displayed as a given intensity of color (Figure \ref{fig:ground_truth_big}d).  In this way, we generated 2-D heatmaps for each hidden and input node in a trained neural network.  This visualization represents the weights connecting a node to output space, or the influence of the activation of that node on the output values (Figure \ref{fig:ground_truth_big}e).  The heatmap visualizations of all nodes' effects in output space is a representation of the causal structure learned by a neural network.

\section{Experiments}
\subsection{Visualizing the Weights of a Neural Network}
In addition to evaluating the different deep learning strategies and distance metric ($d_x$) for finding causal structure, we also needed a method to quickly evaluate many networks without relying on visual inspection.  The solution to this problem came from our method to visualize the weights of a neural network.  Figure \ref{fig:dnn_pr} shows visualizations of the weights for two different DNNs.  Each square in these visualizations represents a specific node in a neural network and the change in the output values induced by the activation of this node (Section \ref{sec:methods_visualizing}).  Figure \ref{fig:dnn_pr}a shows the ground truth patterns/visualizations based on the ground truth DAG (Figure \ref{fig:ground_truth_big}a). These are the patterns we are looking for in the weights of a network to indicate the correct causal structure.  The color of these heatmaps is of minimal importance, while the pattern in the heatmaps is of maximal importance.  

\begin{figure}[tb]
\centering 
\includegraphics[scale=0.92]{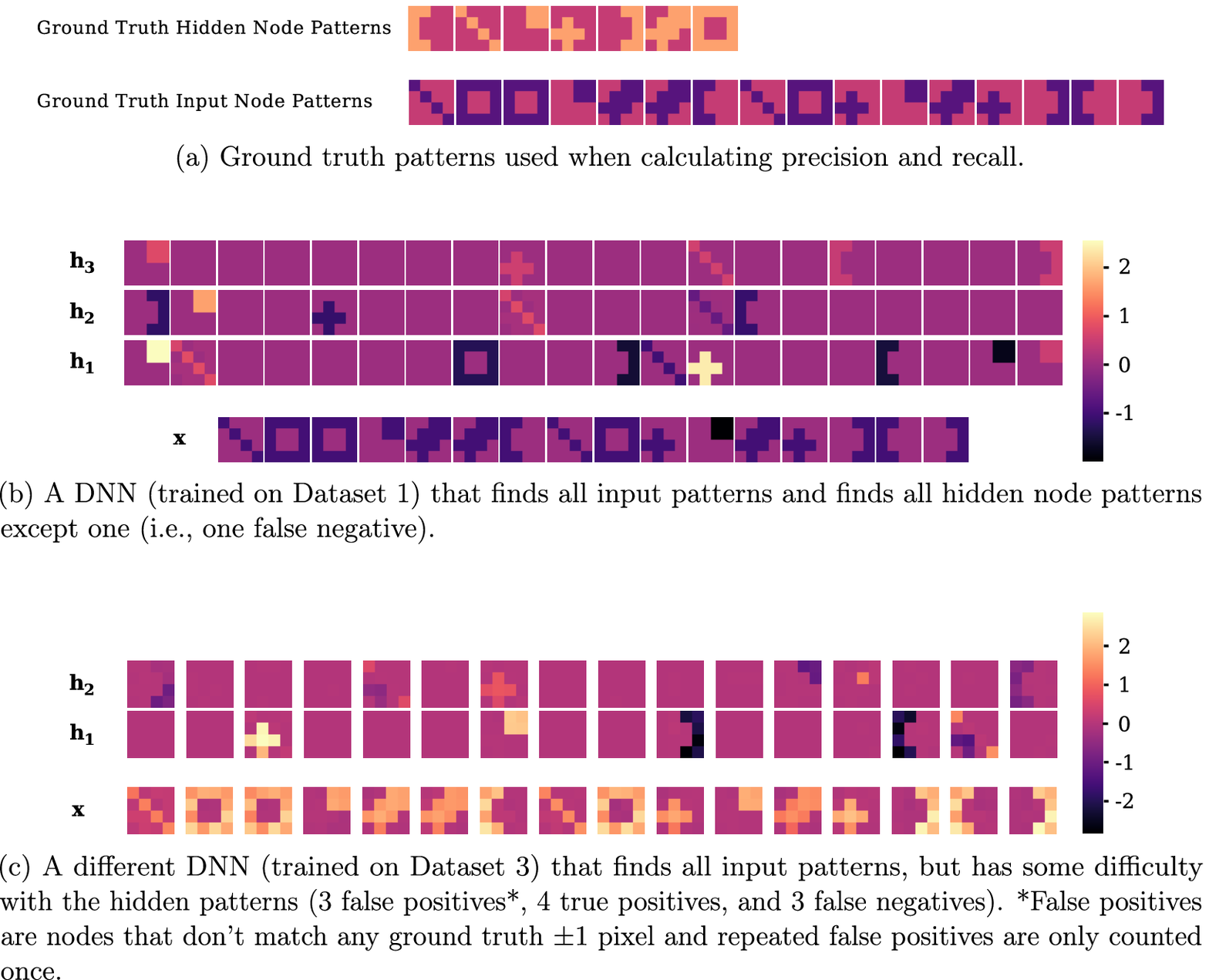}
\caption{DNN weight visualizations.}
\label{fig:dnn_pr}
\end{figure}

Figure \ref{fig:dnn_pr}b shows a visualization of the weights of a 3-hidden layer DNN trained on Dataset 1.  If this DNN had not been regularized, then most of the hidden nodes ($\bm{h_1}$, $\bm{h_2}$, $\bm{h_3}$) in these visualizations would appear mostly as static noise.  The visualizations labeled as $\bm{x}$ show what each of the input nodes for this DNN map to in output space.  If we compare this set of images to the ground truth input node patterns in Figure \ref{fig:dnn_pr}a (bottom), we see that the input nodes in the DNN 100\% match with the ground truth input node patterns.  This example immediately suggests a possible method for evaluating how well neural networks capture causal structure---by calculating precision and recall based on the ground truth patterns that should be present within these visualizations for the causal structure to be correct.

\subsection{Precision and Recall of Causal Structure: Description and Examples}
To calculate precision and recall of the causal structure, we ran the visualization algorithm on all layers in our trained RINN, DNN, DBN, or ES-C, obtaining a vector with 16 values for each node in the neural network, representing the influence of that node on output space.  We took the absolute value and then binarized these vectors in order to match them to the ground truth, thus being able to calculate precision and recall.  This required setting a threshold, above which a value would be set to $1$ and otherwise set to $0$.  We chose this threshold by searching through the ten best models (for each dataset and strategy combination) for a single threshold that provided the best $F_1$ score across all ten models.  The same threshold was then used for all ten models.  We obtained precision and recall measurements for each neural network by comparing these binary vectors to ground truth binary vectors.  To calculate the ground truth vectors, we used the same algorithm as in Section \ref{sec:methods_visualizing} applied to the ground truth binary weight matrices that represent $G_T$.  True positives (TP) here represent ground truth hidden nodes or input nodes (remember, a node is represented as a vector of numbers in output space) that were captured by the weights of the trained network.  False positives (FP) represent predicted vectors from the trained network that did not match any of the ground truth nodes or the zero vector.  False negatives (FN) represent ground truth hidden nodes or input nodes that were not captured by the weights of the trained network.  When comparing a predicted output space vector to a ground truth vector, we considered the predicted to match the ground truth if the binary vectors were identical within $\pm 1$ value (i.e., pixel).  The network in Figure \ref{fig:dnn_pr}b achieved 100\% \textit{input node} precision and recall.

After examining the vector of images labeled as $\bm{h}_1$, in Figure \ref{fig:dnn_pr}b, we see that this hidden layer by itself captures six of the seven ground truth hidden node patterns from Figure \ref{fig:dnn_pr}a ($\bm{h_2}$ and $\bm{h_3}$ do not find any additional ground truth hidden node patterns).  The calculations for \textit{hidden node} precision and recall for this DNN would be as follows:  $Precision (P) = \frac{TP}{TP+FP} = \frac{6}{6+0} = 1.00$, $Recall (R) = \frac{TP}{TP+FN} = \frac{6}{6+1} = 0.86$.  So despite regularization encouraging the DNN to represent patterns as simply as possible, this DNN did not capture the ground truth hidden node that is a combination of the "plus sign" and "little square" (second from the right in Figure \ref{fig:dnn_pr}a top). Nodes 17 and 19 in Figure \ref{fig:ground_truth_big}a and \ref{fig:ground_truth_big}e are referred to as "combination" nodes, as they represent a combination of two other nodes.  Figure \ref{fig:dnn_pr}c shows a visualization of a different DNN that did not capture the causal structure as nicely as Figure \ref{fig:dnn_pr}b. The input nodes map 100\% correctly, but there are at least three FPs and only four TPs found within the hidden layer nodes. $P = \frac{4}{4+3} = 0.57$, $R = \frac{4}{4+3} = 0.57$.

The first hidden layer of a DNN must capture all the hidden node patterns required to completely map inputs to outputs, as it does not have another chance to capture information it may have missed in the first hidden layer later in the network.  The hidden layer representations of a DNN are alternative representations of all the information in the input necessary to calculate the output.  This is in contrast to a RINN, which can simply learn the patterns at a later hidden layer because a RINN has redundant inputs.  So the hidden layers of a RINN are not necessarily alternative representations of the input, but they capture salient aspects of the input data in some way desirable to calculate the output.

\begin{figure}[tb]
\centering 
\includegraphics[scale=0.92]{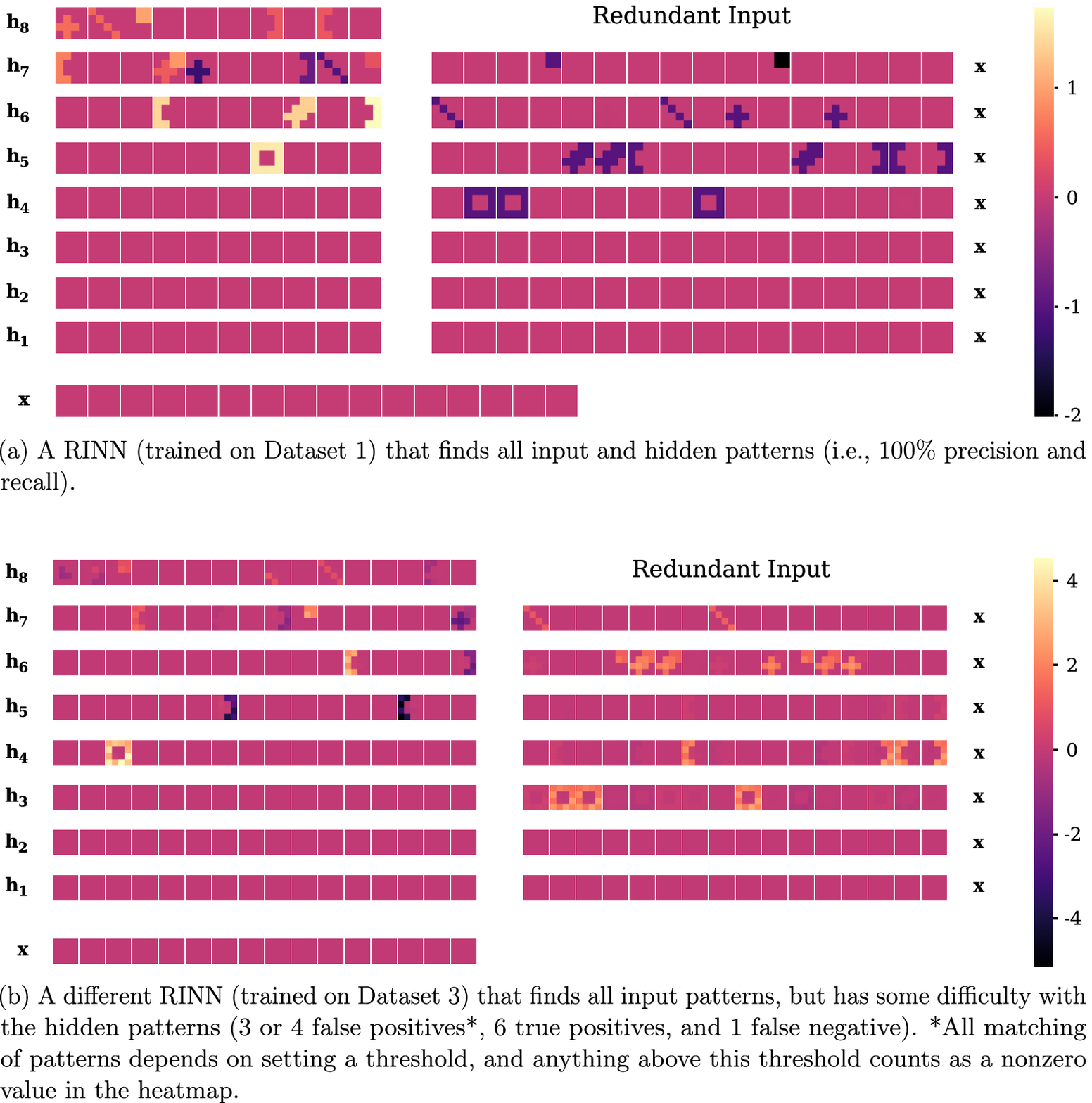}
\caption{RINN weight visualizations.}
\label{fig:rinn_pr}
\end{figure}

Figure \ref{fig:rinn_pr} shows the visualizations for all weight matrices in two different trained RINNs.  The RINN visualized in Figure \ref{fig:rinn_pr}a is an example of a network that scores 100\% precision and recall for both input and hidden nodes.  To evaluate the input nodes of a RINN, all input node heatmaps for a particular input are summed across all redundant layers of the network, and this final sum of heatmaps is compared to the known ground truth for that input.  We can see in Figure \ref{fig:rinn_pr}a that there is no need to sum the heatmaps because the input nodes learn the correct mapping (Figure \ref{fig:dnn_pr}a (bottom)) at only a single redundant input layer (for each input node).  When we look at the heatmaps for the hidden nodes, we see that all seven ground truth patterns from Figure \ref{fig:dnn_pr}a (top) are present in either $\bm{h}_5$, $\bm{h}_6$, $\bm{h}_7$, or $\bm{h}_8$.  Figure \ref{fig:rinn_pr}b shows the visualization of a different RINN that doesn't score as well as Figure \ref{fig:rinn_pr}a.  The input nodes of this RINN achieve 100\% precision and recall.  But the hidden nodes for this RINN show four FPs, six TPs (although one of the combination nodes, i.e., node 19 in Figure \ref{fig:ground_truth_big}c, \ref{fig:ground_truth_big}d, and \ref{fig:ground_truth_big}e, is very close to missing the threshold), and one FN ($7-6 = 1$). The final precision and recall for the hidden nodes would be:  $P = \frac{6}{6+4} = 0.60$, $R = \frac{6}{6+1} = 0.86$.

Precision and recall for the TCGA+OR and TCGA+OR, XOR, AND datasets were calculated based only on the recovery of the simulated data ground truth---as we don't know the ground truth for the TCGA part of the data and therefore have no way to calculate precision and recall.

\subsection{Precision and Recall of Causal Structure: Results}
The major purpose of this work was to compare the ability of different neural network-based strategies to find causal structure in their weights.  We were also interested in the limitations of these strategies depending on the characteristics of the data.  Table \ref{table_pr_input} shows the average \textbf{input} node precision and recall results for all datasets and strategies used in this study.  Table \ref{table_pr_input} also shows the average test set error and AUROC (where appropriate) for all strategies and datasets.  This table is best thought of as showing how well the input mapped to output without indicating the exact causal structure in the weights, i.e., this metric does not indicate how the latent variables interacted to achieve the output.  The values in Table \ref{table_pr_input} (and Table \ref{table_pr_hidden}) represent the average across the ten models with the lowest $d_x$ for each dataset/strategy combination.

\begin{table}[tb]
  \centering
  \resizebox{1.0\columnwidth}{!}{
  \begin{tabular}{l l c c c c c }
    \toprule
    \multicolumn{2}{c}{} & \multicolumn{3}{c}{\large Input Mapping} & \multicolumn{2}{c}{\large Test Set} \\
    \cmidrule(l){3-5} \cmidrule(l){6-7}
    Simulated Data & Strategy & precision & recall & $F_1$ & error & AUROC \\
    \midrule
    
    \multirow{5}{*}{(1) Matrix mult. w/ interv.} & DBN & NA & NA & NA & $0.05 \pm 0.1$ & NA \\
    & DNN & $\bm{1.00} \pm 0.0$ & $\bm{1.00} \pm 0.0$ & $\bm{1.00} \pm 0.0$ & $0.01 \pm 0.0$ & NA \\ 
    & RINN  & $\bm{1.00} \pm 0.0$ & $\bm{1.00} \pm 0.0$ & $\bm{1.00} \pm 0.0$ & $\bm{0.00} \pm 0.0$ & NA \\
    & ES-C & $0.95 \pm 0.1$ & $\bm{1.00} \pm 0.0$ & $0.97 \pm 0.0$ & $0.03 \pm 0.0$ & NA \\
    & DM & $0.31$ & $\bm{1.00}$ & $0.48$ & NA & NA \\
    \midrule
    
    \multirow{5}{*}{(2) Linear Gaussian} & DBN & NA & NA & NA & $\bm{1.82} \pm 0.1$ & NA \\
    & DNN & $0.94 \pm 0.1$ & $0.93 \pm 0.0$ & $0.93 \pm 0.1$ & $9.31 \pm 0.2$ &  NA \\
    & RINN & $\bm{0.98} \pm 0.0$ & $\bm{1.00} \pm 0.0$ & $\bm{0.99} \pm 0.0$ & $9.19 \pm 0.1$ & NA \\
    & ES-C & $0.91 \pm 0.1$ & $0.91 \pm 0.0$ & $0.90 \pm 0.1$ & $10.12 \pm 0.2$ & NA \\
    & DM & $0.69$ & $\bm{1.00}$ & $0.81$ & NA & NA \\
    \midrule
    
    \multirow{5}{*}{(3) OR} & DBN & NA & NA & NA & $\bm{0.55} \pm 0.2$ & $0.68 \pm 0.2$ \\
    & DNN & $\bm{0.99} \pm 0.0$ & $0.98 \pm 0.1$ & $\bm{0.98} \pm 0.0$ & $0.56 \pm 0.0$ & $\bm{0.75} \pm 0.0$ \\
    & RINN & $0.96 \pm 0.1$ & $\bm{1.00} \pm 0.0$ & $\bm{0.98} \pm 0.0$ & $0.56 \pm 0.0$ & $\bm{0.75} \pm 0.0$ \\
    & ES-C & $0.81 \pm 0.2$ & $0.95 \pm 0.1$ & $0.87 \pm 0.1$ & $0.57 \pm 0.0$ & $0.74 \pm 0.0$ \\
    & DM & $0.56$ & $\bm{1.00}$ & $0.72$ & NA & NA \\
    \midrule
    
    \multirow{5}{*}{(4) OR, XOR, AND} & DBN & NA & NA & NA & $0.58 \pm 0.1$ & $0.57 \pm 0.1$ \\
    & DNN & $0.78 \pm 0.2$ & $0.99 \pm 0.0$ & $0.87 \pm 0.1$ & $0.45 \pm 0.0$ & $0.84 \pm 0.0$ \\
    & RINN & $\bm{0.92} \pm 0.1$ & $0.91 \pm 0.1$ & $\bm{0.91} \pm 0.1$ & $0.45 \pm 0.0$ & $0.84 \pm 0.0$ \\
    & ES-C & $0.46 \pm 0.2$ & $0.65 \pm 0.2$ & $0.53 \pm 0.2$ & $\bm{0.41} \pm 0.0$ & $\bm{0.87} \pm 0.0$ \\
    & DM & $0.44$ & $\bm{1.00}$ & $0.61$ & NA & NA \\
    \midrule
    
    \multirow{3}{*}{(5) TCGA+OR} & DBN & NA & NA & NA & $\bm{0.34} \pm 0.2$ & $\bm{0.88} \pm 0.1$ \\
    & DNN & $0.90 \pm 0.1$ & $0.86 \pm 0.2$ & $0.88 \pm 0.1$ & $0.42 \pm 0.0$ & $0.86 \pm 0.0$ \\
    & RINN & $\bm{0.95} \pm 0.1$ & $\bm{0.98} \pm 0.0$ & $\bm{0.96} \pm 0.1$ & $0.40 \pm 0.0$ & $0.87 \pm 0.0$ \\
    \midrule
    
    \multirow{3}{*}{(6) TCGA+OR, XOR, AND} & DBN & NA & NA & NA & $0.53 \pm 0.2$ & $0.64 \pm 0.2$ \\
    & DNN & $\bm{0.79} \pm 0.3$ & $0.90 \pm 0.3$ & $0.84 \pm 0.3$ & $\bm{0.42} \pm 0.0$ & $\bm{0.86} \pm 0.0$ \\
    & RINN & $0.78 \pm 0.1$ & $\bm{0.98} \pm 0.0$ & $\bm{0.87} \pm 0.0$ & $0.43 \pm 0.0$ & $\bm{0.86} \pm 0.0$ \\
    \bottomrule    
  \end{tabular}
  }
  \caption{Average \textbf{input} node precision and recall across ten best models. Do RINN input nodes capture known structure in their weights when mapped to output space?  Best value of the 5 strategies in bold.}
  \label{table_pr_input}
\end{table}

Table \ref{table_pr_input} shows that both DNN and RINN performed quite well at mapping input to output across all datasets, with the RINN performing a little better in terms of precision and recall.  The RINN performed noticeably better than the DNN on the Linear Gaussian and TCGA+OR datasets.  The test errors and AUROCs for DNN and RINN were almost identical.  This means that if one captures more causal structure than the other, it isn't due to one strategy being able to predict better than the other, but rather due to some other difference.  ES-C performed similar to DNN and RINN on the first three datasets, but markedly worse on the OR, XOR, AND dataset in terms of precision and recall.  ES-C had errors and AUROCs that were similar to DNN and RINN (and on Dataset 4 were better than RINN and DNN).  The DM algorithm, despite getting 100\% recall, performed very poorly on all datasets.  This is because the DM algorithm isn't identifying any false-negatives, just a lot of false-positives and maybe 1 or 2 true-positives. 

The results in Table \ref{table_pr_hidden} are the main results of this work and show how well the ground truth causal structure was captured by the \textbf{hidden} nodes of each of the strategies for each dataset.  The RINN had a higher $F_1$ score than all other strategies for most datasets, and consistently had the highest recall across all datasets and strategies.  RINN and DNN had $F_1$ scores within $0.01$ of each other on datasets 4 and 6.  The RINN performed considerably better than any other strategy on the Linear Gaussian dataset.  The Linear Gaussian dataset was also the dataset with the most noise, so it is encouraging that the RINN was able to considerably outperform all other strategies on this noisy dataset, as we suspect that most biological data has a high amount of noise.  There didn't seem to be much of a difference between continuous and binary datasets---the strategies were still able to learn causal structure from completely binary inputs and outputs (Datasets 3 and 4).  Rather, the most important characteristic of the datasets seemed to be the amount of noise.  The RINN outperformed all other strategies on the TCGA+OR dataset and performed similarity to a DNN on the TCGA+OR, XOR, AND dataset. These datasets had inputs and outputs of much higher dimensionality than any of the other simulated datasets, and shows the RINNs utility even in a higher dimensionality setting. 

ES-C performed moderately well, but was outperformed by DNN and RINN on all datasets except Linear Gaussian, where ES-C outperformed DNN.  This suggests that using a parallelized evolutionary algorithm is a direction worth pursuing for finding latent structure, as ES-C performed well despite relatively less model selection, when compared to RINN, after taking into account that ES-C has more hyperparameters to be tuned than RINN.  ES-C was also trained on significantly less data (20\%) than the other strategies to decrease the runtime of ES-C even further.  ES-C wasn't run on datasets 5 and 6 because of the very long running time it had on the small datasets, meaning that running on the much larger datasets would be impractical.  The DM algorithm and DBN performed quite poorly across all datasets, with the DM algorithm performing the worst on most datasets (note that some of the DM algorithms assumptions are violated by $G_T$).  The DM algorithm wasn't run on datasets 5 and 6 because of its poor performance on the smaller datasets.  Overall, the RINN was better able to recover the causal structure in Figure \ref{fig:ground_truth_big}a than any other strategy explored in this work. 

\begin{table}[tb]
  \centering
  \begin{tabular}{l l c c c}
    \toprule
    \multicolumn{2}{c}{} & \multicolumn{3}{c}{Hidden Mapping} \\
    \cmidrule(l){3-5} 
    Simulated Data & Strategy & precision & recall & $F_1$ \\
    \midrule
    
    \multirow{5}{*}{(1) Matrix mult. w/ interv.} & DBN  & $0.73 \pm 0.32$ & $0.61 \pm 0.22$ & $0.66 \pm 0.26$  \\ 
    & DNN & $0.97 \pm 0.07$ & $0.80 \pm 0.07$ & $0.87 \pm 0.06$  \\ 
    & RINN & $\bm{1.00} \pm 0.00$ & $\bm{0.93} \pm 0.08$ & $\bm{0.96} \pm 0.04$ \\ 
    & ES-C & $0.72 \pm 0.08$ & $0.80 \pm 0.13$ & $0.76 \pm 0.08$  \\ 
    & DM & $0.33$ & $0.29$ & $0.31$  \\
    \midrule
    
    \multirow{5}{*}{(2) Linear Gaussian} & DBN & $0.35 \pm 0.13$ & $0.57 \pm 0.17$ & $0.43 \pm 0.14$  \\
    & DNN & $0.34 \pm 0.30$ & $0.43 \pm 0.24$ & $0.37 \pm 0.26$  \\
    & RINN & $\bm{0.65} \pm 0.11$ & $\bm{0.87} \pm 0.18$ & $\bm{0.74} \pm 0.13$ \\ 
    & ES-C & $0.41 \pm 0.10$ & $0.54 \pm 0.06$ & $0.46 \pm 0.09$  \\ 
    & DM & $0.33$ & $0.29$ & $0.31$  \\
    \midrule
    
    \multirow{5}{*}{(3) OR} & DBN & $0.20 \pm 0.30$ & $0.16 \pm 0.23$ & $0.17 \pm 0.25$  \\
    & DNN & $\bm{0.88} \pm 0.17$ & $0.76 \pm 0.15$ & $0.81 \pm 0.14$ \\
    & RINN & $0.87 \pm 0.19$ & $\bm{0.83} \pm 0.13$ & $\bm{0.85} \pm 0.15$  \\
    & ES-C & $0.71 \pm 0.22$ & $0.57 \pm 0.18$ & $0.63 \pm 0.19$  \\ 
    & DM & $0.33$ & $0.29$ & $0.31$  \\
    \midrule

    \multirow{5}{*}{(4) OR, XOR, AND} & DBN & $0.38 \pm 0.22$ & $0.33 \pm 0.19$ & $0.33 \pm 0.17$  \\
    & DNN & $\bm{0.87} \pm 0.18$ & $0.79 \pm 0.17$ & $\bm{0.82} \pm 0.16$  \\
    & RINN & $0.81 \pm 0.17$ & $\bm{0.83} \pm 0.11$ & $0.81 \pm 0.11$  \\
    & ES-C & $0.44 \pm 0.09$ & $0.74 \pm 0.16$ & $0.55 \pm 0.10$  \\ 
    & DM & $0.29$ & $0.29$ & $0.29$   \\
    \midrule
    
    \multirow{3}{*}{(5) TCGA+OR} & DBN & $0.08 \pm 0.17$ & $0.06 \pm 0.14$ & $0.08 \pm 0.15$ \\
    & DNN & $0.29 \pm 0.17$ & $0.59 \pm 0.26$ & $0.33 \pm 0.12$ \\
    & RINN & $\bm{0.37} \pm 0.25$ & $\bm{0.63} \pm 0.24$ & $\bm{0.41} \pm 0.19$ \\
    \midrule
    
    \multirow{3}{*}{(6) TCGA+OR, XOR, AND} & DBN & $0.04 \pm 0.08$ & $0.16 \pm 0.22$ & $0.07 \pm 0.12$ \\ 
    & DNN & $\bm{0.57} \pm 0.26$ & $0.74 \pm 0.16$ & $\bm{0.61} \pm 0.20$ \\
    & RINN & $0.53 \pm 0.24$ & $\bm{0.80} \pm 0.10$ & $0.60 \pm 0.17$ \\
    \bottomrule
    
  \end{tabular}
  \caption{Average \textbf{hidden} node precision and recall across ten best models. Do RINN hidden nodes capture known structure in their weights when mapped to output space?  Best value of the 5 strategies in bold.}
  \label{table_pr_hidden}
\end{table}

\section{Discussion}
The results presented here show that deep learning models (RINN and DNN with $L_1$ weight regularization) can learn latent variables that capture, with high accuracy (precision and recall), the compositional structure generated by causal relationships in simulation experiments.  In doing so, the models are capable of correctly ordering the hierarchical relationships among latent variables that encode the signal of the data-generating process, i.e., the causal relationships among latent variables.  The RINN performed better than the DNN (and all other algorithms) across almost all metrics measured for recovery of the ground truth hidden nodes.  By more accurately capturing the statistical structure of the data-generating process, the RINN can better recover the causal relationships among the latent variables.  These results indicate that allowing inputs to directly access latent variables encouraged the latent variables to capture salient statistical structure connecting input and output variables.  It has not escaped our notice that the RINN architecture simplifies the interpretability of a deep learning model.  The architecture of the RINN allows hidden nodes to be more easily mapped to a set of input variables than a DNN, as inputs in a RINN are \textit{directly} connected to each hidden node.  A set of input variables can then be used to represent a hidden node, and the hidden nodes become partially interpretable.  That is, one can interpret that a latent variable encodes the signal of specific input variables.  These results illustrate two main characteristics (advances) brought forward by the RINN model: the capability of learning a partially interpretable deep learning model and the capability of learning causal relationships.  

We conjecture that the capability of the RINN model to learn causal structure lies in its capability of mimicking the data-generating process.  Here we employed a set of hierarchically organized latent variables to mimic the data-generating process.  We hypothesized that the most efficient way for a model to accurately capture the relationships between input and output is to force latent variables to mimic the data-generating process, and we applied $L_1$ penalization (i.e., following Occam's Razor or the minimal description length principle \citep{rissanen1978modeling}) to maximize the amount of information encoded by each latent variable with respect to both input and output variables.  Interestingly, both RINN and DNN were capable of learning the data-generating process.  Furthermore, if the data-generating process involves a series of causal relationships that lead to compositional statistical structure, the model can naturally capture such information.  Hence, we hypothesize that if a deep learning model is initially set up with an architecture that is sufficient to enclose the causal structures of the data generating process, then one can employ a parameter-searching approach to learn the causal structure among latent variables, in contrast to an explicit search in structure space.  In terms of complexity, searching through the parameter space of a deep learning model can be done with polynomial complexity, whereas searching through the structure space of a set of variables is super-exponential, and greedy algorithms are often used to reduce the complexity.

There are some limitations relevant to this work.  When interpreting the weights of a RINN (or DNN) as a causal graph, it is necessary to look at the values of all the weights in a trained network and, based on a threshold, decide which weights are significant (i.e., edges in a graph) and which are not.  Using $L_1$ regularization causes the network to set most of the values in the weight matrices to small numbers, but in our experiments very few of the weights are actually zero---rather there are many weights with values in the interval $[0.0001, 0.5]$.  This thresholding problem can be improved by using non-differentiable optimization procedures  that constrain the allowed values for weights in a neural network (e.g., ES-C).  Another limitation is that calculating precision and recall did not explicitly take into account the ordering of the nodes.  Anecdotally, after looking through hundreds of trained networks, we very rarely observed incorrect ordering of causal structure.  Using $L_1$ regularization encourages the network to learn nodes in the correct order and hierarchy, as learning $G_T$ represents the weight structure with the lowest number of edges.

The evolutionary strategy explored here provided promising results despite limited model selection and data, and more sophisticated versions are worth exploring in future experiments.  Modifying the RINN by using differential regularization and faster evolutionary algorithms, as well as extending the causal framework developed here could lead to improvements in the interpretability of deep neural networks and in the use of deep learning models for the discovery of latent causal structure.


\acks{Funding for this work was provided by NIH grant R01 LM012011. Author BA was supported by training grant T15LM007059 from the NLM.} 


\newpage

\appendix
\section*{Appendix A. Model Selection}


\label{appendix_ms}
For the deep learning strategies, each simulated dataset was separated into training, validation, and test sets according to the following ratio 0.75/0.15/0.10, respectively. We wanted to avoid the overfitting that often occurs if using a single validation dataset, so using the 90\% (training + validation) split, we randomly selected 15\% for one validation set (while training on the 75\% remaining data) and a different 15\% for the other validation set (while again training on the remaining 75\%), for a total of two validation sets.  We used a combined random and grid search approach \citep{bengio2012practical, hinton2012practical} to select the hyperparameters (learning rate, regularization rate, size of hidden layer, etc.) for the best models with the main objective of finding the optimal balance between sparsity and prediction error to encourage the discovery of latent causal structure. See Figure \ref{fig:ms} for example model selection results. 

Our search through hyperparameter space varied somewhat depending on the data and the deep learning strategy being evaluated.  Datasets 5 and 6 were considerably larger than the other simulated datasets as they contained a large amount of real TCGA data.  This increased the number of hyperparameters to be evaluated (i.e., more hidden layer sizes to search through because the input was larger), so we performed model selection on 100 more sets of hyperparameters for these larger datasets. Also, DNNs, DBNs, and ES-C have additional hyperparameters to search through relative to the RINN.  For DNNs and DBNs, we need to find the optimal number of hidden layers (whereas the structure of the RINN allows one to largely avoid determining the ideal number of hidden layers---Section \ref{sec:strategies_rinn}).  For simplicity, we assumed all DNN hidden layers to be of the same size.  However, DBN model selection uses the decreasing nature of the hidden layers as a driving force to learn the most salient aspects of the data.  Therefore, DBN model selection included searching over different sizes of the hidden layers for each of the hidden layers.  This means more hyperparameters to search through.  The additional hyperparameters for ES-C included, elite ratio, mutation rate, and set of legal weight values.

For binary outputs, binary cross-entropy plus $L_1$ regularization was used as the objective function.  For non-binary outputs, mean squared error (MSE) plus $L_1$ regularization was used.  Almost all of the best performing models used ReLU (Rectified Linear Unit) activation functions, although we performed model selection using sigmoid and tanh activation functions as well.

A difference between the RINN and a DNN in this paper is that the RINN was always trained with eight hidden layers.  We hypothesize (for eventually using this in a biological setting) that most biological pathways will not have more than an 8-level hierarchy.  Importantly, since copies of the input are available to use at each hidden layer (should the algorithm decide to use it), we don't need to determine the best number of hidden layers (as a model selection hyperparameter)---this means less model selection is required for RINNs.  If the optimal way to learn the structure in the data with a RINN is by using a 3-hidden layer network, the RINN will learn nonzero weights starting with the copy of the input that is connected to hidden layer 6.  This would create a 3-hidden layer network (hidden layers 6, 7, and 8), with all weights in the network before the redundant input connected to hidden layer 6 being equal to 0.0.  This happens because of $L_1$ regularization, which constrains the RINN to map input to output using as few weights as possible.  If there is reason to think that there may be more hierarchical levels in a particular dataset, the number of maximum hidden layers can always be increased with a RINN or treated as a hyperparameter and tuned using model selection.

For ES-C, we performed model selection over multiple hyperparameters including, elite ratio (percent of population saved to mate and produce the next generation), mutation rate, regularization rate, and the legal weight values.  We used a population size of 200 and evolved for up to 13,000 generations or epochs.  The \textit{MATE} function involved choosing two individuals in the elite population and randomly combining 50\% of the weights from one individual with 50\% of the weights from the other individual (Algorithm  \ref{alg:ga}).  The evolutionary strategy used in this work was not set up to be parallelized across processors, as early prototyping experiments suggested that ES-C would not perform nearly as well as RINN.  This, unfortunately, reduced the amount of model selection we could perform in a timely manner.  To speed up ES-C, we only trained on 20\% of the data that was used to train the other algorithms.  Testing 500 sets of hyperparameters on one dataset took approximately three weeks on a single desktop computer.  See Table \ref{table_ms} for a breakdown of the number of sets of hyperparameters that were evaluated for each strategy. 

\begin{table}[tb]
  \centering
  \begin{tabular}{ l c c }
    Strategy & Datasets 1,2,3,4 & Datasets 5,6 \\
    \hline
    RINN & 440 & 540  \\
    DNN & 732 & 832  \\
    DBN & 940 & 1040  \\
    ES-C & 500 & NA \\
    \hline
  \end{tabular}
\caption{Number of sets of hyperparameters evaluated for each deep learning strategy.}
 \label{table_ms}
\end{table}

\begin{figure}[tb]
\centering 
\includegraphics[scale=0.92]{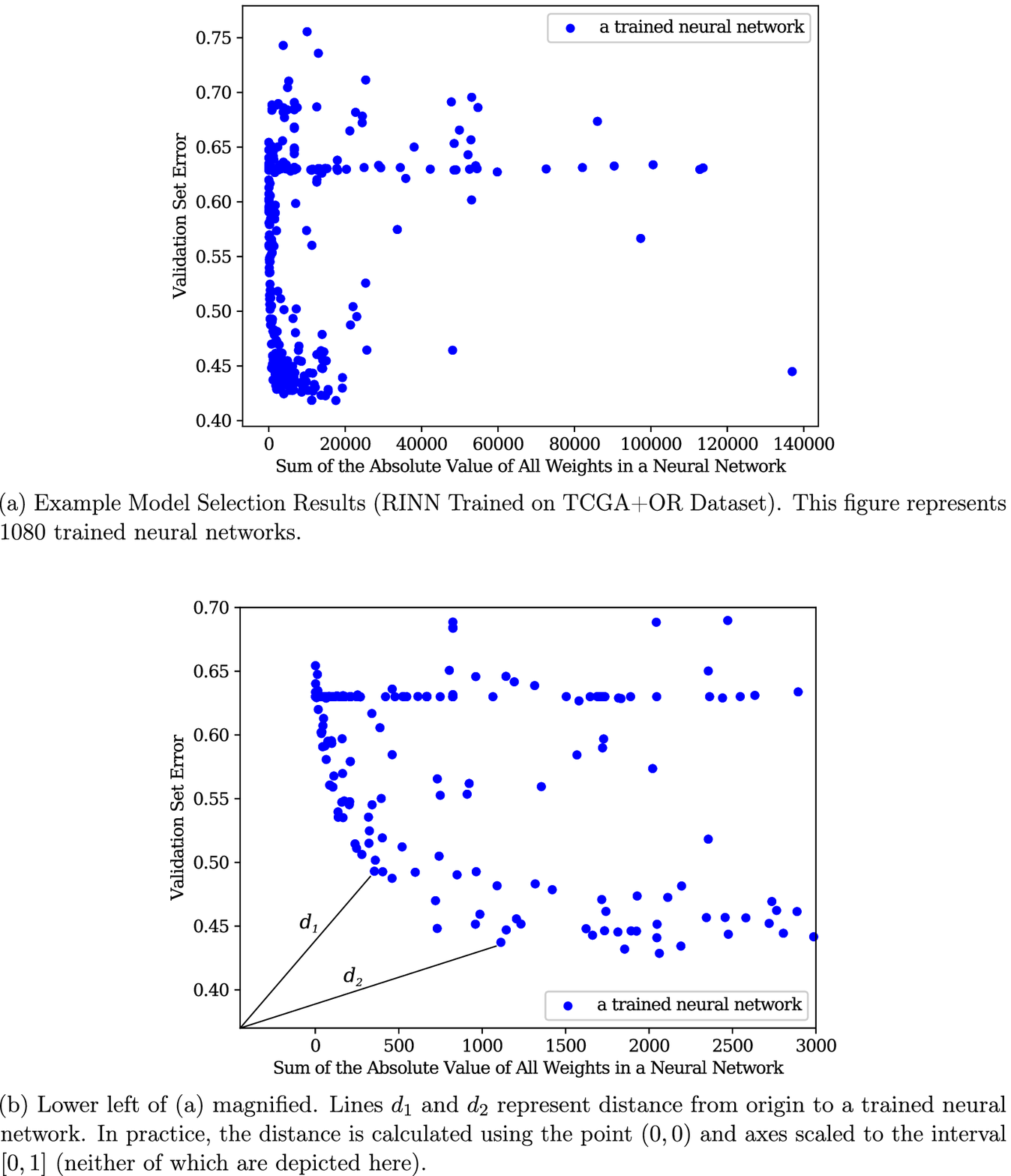}
\caption{Measuring euclidean distance to rank model selection results.}
\label{fig:ms}
\end{figure}

Figure \ref{fig:ms} shows model selection results for RINNs on simulated Dataset 5. These are typical model selection results across all deep learning strategies and datasets. Figure \ref{fig:ms}a shows the validation set error and the sum of the absolute value of all weights in a network for 1,080 (540 times two validation sets) fully trained RINNs. There is a concentration of networks (shown as blue circles) in the lower left corner of this figure. It is within these models that we hypothesize one will find the models with the highest probability of capturing the causal structure of the data within their weights. Figure \ref{fig:ms}b shows a magnification of the lower left of Figure \ref{fig:ms}a, and here it becomes apparent that model selection based solely on prediction error or sparsity (sum of the absolute values of the weights) is inadequate. We also don't know, 1. Will deep learning strategies capture causal relationships in their weights? and 2. If they do capture this structure in their weights, how can we identify them during model selection (i.e., Where would these models be located in the plots in Figure \ref{fig:ms}?)? We hypothesized that the models with best chances of capturing causal relationships in their weights will be the models that optimally balance both prediction error and sparsity (i.e., the models in bottom left of Figure \ref{fig:ms}a). To quantify this hypothesis and provide a ranking of the most optimal models, we measured the distance between each model's point location in the plot in Figure \ref{fig:ms}b and the origin, and used this as a metric to rank the models (Section \ref{sec:methods_distance} for more details). We hypothesize that the models closer to the origin (using the distance metric described above) would better capture the ground truth causal structure in their weights.

\section*{Appendix B. DM Algorithm}
Sober's criterion is used to determine the presence of latent variables between inputs and outputs (i.e., causes and effects) when inputs are known to cause outputs.  If two outputs are independent of one another when conditioned on their corresponding causal inputs, then there is no latent variable between these inputs and outputs.  If not conditionally independent, then there is a latent variable present \citep{murray2014dm}.  

The assumptions required by the DM algorithm are somewhat similar to the assumptions required for the RINN.  There are a few differences between the assumptions made by each algorithm. The DM algorithm requires each latent variable to have have at least one observed cause and one observed effect; however, the RINN can represent a latent variable without an observed cause through the use of a nonzero bias node. The DM algorithm also assumes that there is only one directed path between any two observed variables, meaning that each output variable has only one directed edge into it. \textit{This is a very limiting assumption for our purposes and the RINN algorithm does not assume this}. Biologically, this assumption is violated as there is often redundancy in cellular signaling pathways and many ways to activate a given protein (e.g., transcription factor). The DM algorithm is also dependent on conditional independence tests, which have additional hyperparameters that need to be set. In general, the DM algorithm has more constraints than the RINN algorithm. However, the correctness of the RINN is dependent upon reaching the global optimum, which is most often not the case when performing gradient descent to optimize neural networks.  Another major difference between the RINN and the DM algorithm is that the DM algorithm only returns a causal structure (without parameterization), whereas the RINN returns a causal structure and parameterization. The DM algorithm is available in Tetrad (http://www.phil.cmu.edu/tetrad/, https://github.com/cmu-phil/tetrad).

\section*{Appendix C. Pseudocode}

\begin{algorithm}
\small
\caption{Simulate Data using Matrix Multiplication and Interventions}
\label{alg:methods_matrix_mult}
\begin{algorithmic}
\State \textit{p} is the Bernoulli success probability \Comment probability that mutation is present
\State \textit{n} is the number of samples we want to generate
\State $\bm{Data}$ is a container for the simulated data
\State $\bm{W}_l$ is the $G_T$ weight matrix between hidden layer $l-1$ and $l$
\State $\bm{h}_l$ is the vector of values representing hidden layer $l$
\State $\bm{W}_x$ is a ground truth matrix with all input node adjacencies
\State
\State $p = 0.10$
\For{$i=1$ to $n$}
	\State SAMPLE $\bm{x} \sim \mathcal{B}(16, p)$ \Comment Bernoulli Trials
    \State $\bm{h}_1 = [1,1,1,1]^\top$
    \For{$j=1$ to $2$}
    	\State set nodes in $\bm{h}_j$ targeted by $\bm{x}$ to $0$ \Comment use $\bm{W}_x$
        \State $\bm{h}_{j+1} = \bm{h}_j \bm{W}_{j+1}$
    \EndFor
    \State SAVE $\bm{x}$ and $\bm{h}_{j+1}$ in $\bm{Data}$
\EndFor
\State $return$ $\bm{Data}$
\end{algorithmic}
\end{algorithm}

\begin{algorithm}
\small
\caption{Simulate Data from AND, OR, XOR Logical Operators}
\label{alg:and_or_xor}
\begin{algorithmic}
\State $G_T = (V,E)$ \Comment ground truth DAG
\State PARENTS($node$) returns the binary values of the parents of $node$
\State $\bm{X}$ is an $n \times |V|$ matrix 
\State $\bm{Data}$ is an $n \times |V|$ matrix 
\State $\bm{D}$ is an associative array or dictionary of key:value pairs, with values that are Bernoulli distributions
\State \textit{n} is the number of samples we want to generate

\State
\For{$node$ in $V$} 
    \State $op = $ RANDOM\_CHOICE($AND, OR, XOR$)
	\For{each possible binary combination, $b$, of parents of $node$}
	    \State $key=node, b$
		\If{$op(b)$ is $True$}
	        \State SAMPLE $\bm{D}[key] \sim Beta(95,5)$ \Comment add to container $\bm{D}$
        \ElsIf{$op(b)$ is $False$}
	        \State SAMPLE $\bm{D}[key] \sim Beta(5,95)$ \Comment add to container $\bm{D}$
        \EndIf
	\EndFor
\EndFor
\State
\State SAMPLE $\bm{X} \sim \mathcal{U}(0,1)$
\For {$i=1$ to $n$}
    \For {$node$ in $V$ } \Comment input (no ancestors) to output; need parents to calc. child
        \State $b = $ PARENTS($node$)
        \State $data_{i,node} = \bm{D}[node,b] < x_{i,node}$ \Comment index into $\bm{X}$ to get $x_{i,node}$
    \EndFor
\EndFor
\State $return$ $\bm{Data}$ \Comment binary matrix
\end{algorithmic}
\end{algorithm}

\newpage

\vskip 0.2in
\bibliography{bibliography}

\end{document}